\documentclass[11pt,a4paper]{article}
\usepackage[margin=25mm]{geometry}
\usepackage{fontspec}
\usepackage{amsmath,amssymb,amsthm,mathtools}
\usepackage{unicode-math}
\usepackage{booktabs,array,tabularx,longtable,graphicx,tikz,placeins,hyperref,fancyhdr}
\hypersetup{colorlinks=true,linkcolor=blue!45!black,urlcolor=blue!45!black,citecolor=blue!45!black,pdftitle={Generalized DCCQ: From Binary Quotients to Multinomial Simplex Geometry and Critical-Strip Coordinates},pdfauthor={Y. Kenan Yılmaz}}
\newtheorem{theorem}{Theorem}[section]
\newtheorem{proposition}[theorem]{Proposition}
\newtheorem{corollary}[theorem]{Corollary}
\theoremstyle{definition}\newtheorem{definition}[theorem]{Definition}
\newcommand{\D}{\Delta}\newcommand{\R}{\mathbb R}\newcommand{\Q}{\mathbb Q}\newcommand{\N}{\mathbb N}\newcommand{\C}{\mathbb C}\newcommand{\Strip}{\mathcal S}
\title{Generalized DCCQ: From Binary Quotients to Multinomial Simplex Geometry and Critical-Strip Coordinates}
\author{Y. Kenan Yılmaz}
\date{15 September 2026}
\begin{document}\maketitle
\begin{abstract}
We extend the discrete complex complement quotient (DCCQ) framework from binary Bernoulli counts to multinomial count compositions. For $m+1$ categories, $m$ is the number of independent probability degrees of freedom. Integer count vectors modulo common scaling determine rational points of the $m$-dimensional probability simplex. Building on standard simplex and log-ratio coordinate geometry, for $m\ge2$ we define the full multinomial DCCQ coordinate map
\[
\Phi_m(\mathbf p)
\coloneqq
\left(
 p_0+i\log\frac{p_1}{p_m},
 \log\frac{p_2}{p_m},
 \ldots,
 \log\frac{p_{m-1}}{p_m}
\right),
\]
and show that it is a real-analytic diffeomorphism
\[
\D_m^\circ\cong\Strip\times\R^{m-2},
\qquad
\Strip\coloneqq\{s\in\C:0<\Re s<1\}.
\]
The previously established binary baseline $m=1$ gives a critical-line coordinate, while the ternary case $m=2$ gives the full open critical strip; higher multinomial models retain $m-2$ additional real contrasts. We also give a one-versus-rest specialization, an exact integer-lattice realization of the ternary coordinate, and a hyperbolic representation of its log-ratio. No zero-location theorem or proof of the Riemann Hypothesis is claimed.
\end{abstract}

\section*{Notation at a glance}
Bold symbols denote vectors; their components are scalar. The symbol $\coloneqq$ introduces a definition, and all logarithms are natural. Only the main recurring symbols are listed here; Appendix~\ref{app:notation} gives the full notation reference.
\begin{center}
\small
\renewcommand{\arraystretch}{1.18}
\begin{tabularx}{\linewidth}{@{}p{.24\linewidth}X@{}}
\toprule
Symbol & Meaning \\
\midrule
$m$ & Probability dimension; the alphabet has $m+1$ categories. \\
$\mathbf p,\ \mathbf k$ & Probability and integer count vectors, each with $m+1$ components. \\
$n,\ [\mathbf k]$ & Total count $n\coloneqq\sum_jk_j$; the class of counts with the same normalized composition. \\
$\D_m,\ \D_m^\circ$ & The $m$-dimensional probability simplex and its relative interior. \\
$\Strip,\ s=\sigma+it$ & The open critical strip $0<\sigma<1$ and its complex coordinate. \\
$\Phi_m$ & Full multinomial DCCQ coordinate map, for $m\ge2$. \\
$R_m$ & One-versus-rest DCCQ map to the line $\Re s=1/(m+1)$. \\
\bottomrule
\end{tabularx}
\end{center}

\section{Introduction}\label{sec:intro}
The ordered Bernoulli-word kernel
\begin{equation}B(p;n,k)\coloneqq p^k(1-p)^{n-k},\qquad 0<p<1,\quad n\in\N,\quad 0\le k\le n,\quad k\in\mathbb Z,\tag{1}\end{equation}
contains a quotient geometry in which a count pair is reduced to its normalized proportion and complement exchanges the two count coordinates. Earlier work used ordered Bernoulli levels and complement-preserving complex continuation to obtain the vertical geometry $\Re s=1/2$ \cite{yilmaz2026dccq}. The present paper asks what changes when the alphabet itself is enlarged from the binary alphabet $\{0,1\}$ to an alphabet $\{0,1,\ldots,m\}$ with $m+1$ categories. We use this indexing only for notational convenience: after the normalization constraint, the corresponding probability simplex has exactly $m$ independent degrees of freedom.

There are two logically different directions of generalization. One may retain the binary alphabet and enlarge exponents, levels, or complex parameters; or one may enlarge the alphabet and replace the Bernoulli pair by an $(m+1)$-category composition. The first retains a one-dimensional probability geometry. The second replaces the Bernoulli interval by the $m$-dimensional probability simplex, denoted $\D_m$, and is the subject here.

The count-to-simplex passage is established. In compositional-data analysis, closure divides a positive vector by its total; vectors on the same positive ray map to the same simplex point, and log-ratio transformations provide unconstrained real coordinates \cite{aitchison,egozcue}. Lovell, Chua, and McGrath explicitly describe natural-number lattice compositions, with simplex points representing equivalence classes of vectors on rays from the origin \cite{lovell}. We take this established count-ray/simplex geometry as the starting point for the DCCQ extension developed below.

The question addressed here is therefore not whether a multinomial simplex admits coordinates---this is standard---but how the binary DCCQ construction extends while retaining all multinomial degrees of freedom, and what complex-strip geometry results from that extension.

The proposed contribution is the subsequent DCCQ complex coordinate construction and symmetry interpretation. Four structures are central:
\begin{enumerate}
\item a full $m$-degree-of-freedom coordinate map $\Phi_m:\D_m^\circ\to\Strip\times\R^{m-2}$ for $m\ge2$, where $\Strip\coloneqq\{s\in\C:0<\Re s<1\}$ is the open critical strip;
\item a barycentrically anchored one-dimensional one-versus-rest Bernoulli submodel, mapped by $R_m$ to the vertical line $\Re s=\frac{1}{m+1}$;
\item the ternary specialization ($m=2$), whose full open simplex is real-analytically diffeomorphic to the open critical strip;
\item an integer-lattice realization of this coordinate chart on rational simplex points.
\end{enumerate}
The earlier binary work~\cite{yilmaz2026dccq} supplies the critical-line baseline. In the present coordinate convention, this geometry is parameterized by the elementary logit map $R_1$; the multinomial development begins at $m=2$. The dimension ladder is thus
\[
\begin{array}{ccl}
m=1&:&\text{binary interval}\longrightarrow\text{one critical line},\\
m=2&:&\text{ternary simplex}\longleftrightarrow\text{full critical strip},\\
m>2&:&\text{multinomial simplex}\longleftrightarrow\Strip\times\R^{m-2}.
\end{array}
\]
The first row recalls that binary geometry, using $R_1$ as its parameterization here. The multinomial development of the present paper starts at $m=2$, where the ternary model has exactly two real degrees of freedom and therefore admits a full complex-strip coordinate under $\Phi_2$. For $m>2$, the full coordinate map $\Phi_m$ retains the additional $m-2$ real contrasts. The ternary case is distinguished dimensionally, not arithmetically. The recurring value $\frac{1}{3}$ is a barycentric reference for the first non-binary alphabet, not an intrinsically privileged strip coordinate. Other vertical lines arising from the DCCQ construction are DCCQ characteristic lines; the term \emph{critical line} is reserved for $\Re s=1/2$. Membership in a line does not identify a zero of $\zeta$ or $\xi$.

For a first reading, Section~2 gives the direct multinomial DCCQ construction without requiring simplex theory. Section~3 then formalizes the count-ray/simplex identification and proves the full coordinate-chart theorem. The later sections develop the one-versus-rest family, the ternary specialization, lattice realization, the hyperbolic representation, and symmetry consequences.

\section{Multinomial DCCQ construction}\label{sec:construction}
\subsection{Ordered multinomial kernel}\label{sec:kernel}
Let $m\ge1$. We begin with an alphabet of $m+1$ categories. Its probability and count vectors are written
\[
\mathbf p\coloneqq(p_0,p_1,\ldots,p_m),\qquad
\mathbf k\coloneqq(k_0,k_1,\ldots,k_m),
\]
with $p_j\ge0$, $\sum_{j=0}^{m}p_j=1$, $k_j\in\N_0$, and positive total $n\coloneqq\sum_{j=0}^{m}k_j$. The normalization constraint leaves exactly $m$ independent probability degrees of freedom. The ordered multinomial word kernel is
\begin{equation}M(\mathbf p;\mathbf k)\coloneqq\prod_{j=0}^{m} p_j^{k_j}.\tag{2}\label{eq:kernel}\end{equation}
A zero-exponent factor is defined to be $1$, including when its probability is zero. We omit the multinomial multiplicity coefficient because the object is the probability of an individual ordered word, not the total probability of all words with a given count vector. The count-ray and simplex interpretation is postponed to Section~\ref{sec:simplex}.

\paragraph{Separate the distinguished mass from the remaining composition.}
Assume now that all $p_j>0$. Category $0$ has mass $p_0$; the remaining categories have total mass $1-p_0$. Their conditional probabilities are
\[
q_j\coloneqq\frac{p_j}{1-p_0},\qquad 1\le j\le m,
\qquad \sum_{j=1}^{m}q_j=1.
\]
Substituting $p_j=(1-p_0)q_j$ into \eqref{eq:kernel} gives the exact factorization
\[
M(\mathbf p;\mathbf k)
=
\underbrace{p_0^{k_0}(1-p_0)^{n-k_0}}_{\text{distinguished category versus the rest}}
\underbrace{\prod_{j=1}^{m}q_j^{k_j}}_{\text{composition within the rest}}.
\]
This elementary conditional decomposition does not restrict the model: the $q_j$ remain free subject to their sum constraint. In particular, they have not been made equal. The binary complement survives as the split $p_0$ versus $1-p_0$; in the multinomial setting, the complementary mass $1-p_0$ is further resolved by the conditional composition $(q_1,\ldots,q_m)$. The kernel thus separates one mass parameter $p_0$ from $m-1$ conditional-composition parameters. The fixed uniform split used by the one-versus-rest submodel will be imposed only in Section~\ref{sec:ovr}.

\subsection{Full multinomial DCCQ coordinate}\label{sec:coordinate}
For $m\ge2$, choose the final conditional category as reference and define
\begin{equation}
u_j\coloneqq\log\frac{q_j}{q_m}=\log\frac{p_j}{p_m},\qquad 1\le j\le m-1.\tag{3}\label{eq:contrasts}
\end{equation}
The common mass $1-p_0$ cancels inside these ratios, so the contrasts describe the composition within the rest, independently of its total mass. The DCCQ coordinate construction retains that mass through $p_0$, uses one oriented contrast as the imaginary part, and keeps every remaining contrast as a real coordinate.
\begin{definition}[Full multinomial DCCQ coordinate map]
For $m\ge2$, define
\begin{equation}
\Phi_m(\mathbf p)\coloneqq\left(p_0+i\log\frac{p_1}{p_m},\log\frac{p_2}{p_m},\ldots,\log\frac{p_{m-1}}{p_m}\right)\in\Strip\times\R^{m-2},\tag{4}\label{eq:fullmap}
\end{equation}
where $\Strip$ is the open critical strip and there are no additional real coordinates when $m=2$.
\end{definition}
The normalized $(m+1)$-category model has $m$ real degrees of freedom. One mass and one contrast form the complex coordinate, while the remaining $m-2$ contrasts stay explicit. Thus $m=2$ is the first case in which no residual real coordinate remains. The one-dimensional binary case $m=1$ is represented separately by $R_1$, which parameterizes the binary baseline.

The kernel can also be written in precisely these real coordinates. Define
\[
Z\coloneqq1+\sum_{j=1}^{m-1}e^{u_j}.
\]
Since $q_j=e^{u_j}/Z$ for $j<m$ and $q_m=1/Z$, the factorization above becomes
\[
M(\mathbf p;\mathbf k)
=
p_0^{k_0}(1-p_0)^{n-k_0}
\frac{\exp\!\left(\sum_{j=1}^{m-1}k_j u_j\right)}{Z^{\,n-k_0}}.
\]
Thus the kernel supplies a statistical interpretation of the coordinates: a distinguished mass and conditional log-ratios. Combining $p_0$ with $u_1$ into a complex number is an additional coordinate choice, not a uniquely forced consequence of the kernel. The inverse-map proof in Section~\ref{sec:chart} establishes that this choice loses no probability information. No simplex theory is needed to follow the construction so far.

\section{Count rays and simplex geometry}\label{sec:simplex}
This section formalizes the geometric meaning of the construction above. The count-ray quotient is first identified with the probability simplex; the DCCQ coordinate $\Phi_m$ is then recognized as a full coordinate chart on its interior.

\subsection{Count-ray quotient and simplex reconstruction}
The normalized probability domain used in Section~2 is the $m$-dimensional probability simplex
\[
\D_m\coloneqq\{\mathbf p\in\R^{m+1}:p_j\ge0,\ \sum_{j=0}^{m}p_j=1\},
\]
with relative interior $\D_m^\circ$ obtained by requiring all $p_j>0$. Here relative interior is taken in the affine hyperplane $\sum_{j=0}^m p_j=1$, not in the ambient space $\R^{m+1}$.

Use nonzero count vectors as the discrete objects:
\[
\mathcal{C}_m\coloneqq\N_0^{m+1}\setminus\{\mathbf0\}.
\]
The total $n=\sum_jk_j$ is determined by $\mathbf k$ and is not an independent datum. Define the count-ray equivalence relation by
\begin{equation}
\mathbf k\sim\mathbf k'
\quad\Longleftrightarrow\quad
\frac{\mathbf k}{\sum_{j=0}^m k_j}
=
\frac{\mathbf k'}{\sum_{j=0}^m k'_j}.
\tag{5}\label{eq:count-equivalence}
\end{equation}
\begin{theorem}[Count-ray simplex reconstruction]\label{thm:count}
The quotient $\mathcal{C}_m/\!\sim$ is naturally identified with $\D_m\cap\Q^{m+1}$. Under this identification, its Euclidean closure is the full closed simplex $\D_m$. For strictly positive counts, the quotient is $\D_m^\circ\cap\Q^{m+1}$, dense in the open simplex.
\end{theorem}
\begin{proof}
Every $\mathbf k/n$ has nonnegative rational coordinates summing to one. Conversely, choose a common denominator $n$ for a rational simplex point $\mathbf r$ and set $k_j\coloneqq nr_j$. Equality of normalized vectors is precisely the equivalence relation. Rational density gives the closure and interior-density statements.
\end{proof}
In the ternary case, the ray represented by $\mathbf k=(2,3,1)$ also contains $2\mathbf k=(4,6,2)$; both normalize to $(1/3,1/2,1/6)$. This same ray will reappear below as a local check of the lattice and hyperbolic representations.

This is the $(m+1)$-category form of the binary quotient and agrees with lattice-composition closure geometry \cite{lovell}. The DCCQ-specific step is the complex coordinate construction already introduced in Section~2.

\paragraph{The information removed by normalization.}
Every nonzero integer count vector has the unique decomposition
\[
\mathbf k=g\,\mathbf k_{\mathrm{prim}},\qquad g\coloneqq\gcd(k_0,\ldots,k_m),\qquad\gcd(\mathbf k_{\mathrm{prim}})=1.
\]
The normalized rational point determines $\mathbf k_{\mathrm{prim}}$ uniquely; $g$ restores the original count state and its total. Thus normalization retains composition and forgets common integer scale. The present coordinate construction works at the ray level. Scale-dependent counting may be added as a separate structure without changing this quotient.

\paragraph{Scale in the kernel and in the quotient.}
For positive $\mathbf p$ and a positive integer $c$,
\[
M(\mathbf p;c\mathbf k)=M(\mathbf p;\mathbf k)^c,
\qquad
\frac{1}{n}\log M(\mathbf p;\mathbf k)
=\sum_{j=0}^m\frac{k_j}{n}\log p_j.
\]
The normalized log-kernel depends on the empirical composition $\mathbf k/n$, whereas the raw kernel is not invariant under count scaling. A probability parameter $\mathbf p$ in $M$ need not equal the empirical composition; their identification below is used only when evaluating the coordinate map on a count class.

\subsection{The DCCQ coordinate as a simplex chart}\label{sec:chart}
Classical additive log-ratio coordinates already identify the open probability simplex with Euclidean space \cite{aitchison1982,aitchison,egozcue}. The map $\Phi_m$ is an explicit real-analytic reparameterization of this standard coordinate geometry: it retains one bounded mass coordinate, combines one oriented log-ratio with that mass coordinate as a complex variable, and leaves the remaining log-ratios as real coordinates.

More specifically, before the complex coordinate construction, the real tuple $(p_0,u_1,\ldots,u_{m-1})$ is, after restriction to the probability simplex, an instance of the standard mixed-coordinate construction of information geometry \cite{sei}. In the representation
\[
\psi(\mathbf p)=B\log\mathbf p,\qquad
\lambda(\mathbf p)=A\mathbf p,\qquad
AB^{\mathsf T}=0,
\]
take $A$ with rows $\mathbf 1^{\mathsf T}$ and $\mathbf e_0^{\mathsf T}$, and take the rows of $B$ to be $(\mathbf e_j-\mathbf e_m)^{\mathsf T}$ for $1\le j\le m-1$. On the simplex this gives $\lambda=(1,p_0)$ and $\psi=(u_1,\ldots,u_{m-1})$. Thus neither this real coordinate system nor the underlying simplex-to-Euclidean equivalence is claimed as new. The DCCQ-specific step considered here is the subsequent complex coordinate construction $p_0+i u_1$, together with its dimensional and reflection/symmetry interpretation.

With the simplex identification now explicit, $\Phi_m$ becomes a global real-analytic chart.

\begin{theorem}[Full $m$-degree-of-freedom coordinate theorem]\label{thm:full}
For every $m\ge2$, $\Phi_m$ is a real-analytic diffeomorphism
\begin{equation}\D_m^\circ\cong\Strip\times\R^{m-2}.\tag{6}\label{eq:diffeomorphism}\end{equation}
Equivalently, before the complex coordinate construction, $\D_m^\circ\cong(0,1)\times\R^{m-1}$.
\end{theorem}
\begin{proof}
Write target coordinates as $(s,u_2,\ldots,u_{m-1})$, with $s=\sigma+iu_1$ and $0<\sigma<1$. Using $Z=1+\sum_{j=1}^{m-1}e^{u_j}$ from Section~\ref{sec:coordinate}, the inverse is
\begin{equation}
p_0=\sigma,\quad p_j=(1-\sigma)\frac{e^{u_j}}{Z}\ (1\le j\le m-1),\quad p_m=(1-\sigma)\frac{1}{Z}.\tag{7}\label{eq:inverse}
\end{equation}
The coordinates are positive, sum to one, and recover all the prescribed log-ratios. Conversely, substituting the coordinates of any $\mathbf p\in\D_m^\circ$ into this formula recovers $\mathbf p$. Both directions are real analytic; $\Strip\times\R^{m-2}$ is regarded as an open real $m$-dimensional domain.
\end{proof}
Through Theorem~\ref{thm:count}, the notation on count classes means
\[
\Phi_m([\mathbf k])\coloneqq\Phi_m\!\left(\frac{\mathbf k}{n}\right),
\qquad n=\sum_{j=0}^m k_j,
\]
for strictly positive counts. Substitution gives
\begin{equation}
\Phi_m([\mathbf k])=\left(\frac{k_0}n+i\log\frac{k_1}{k_m},\log\frac{k_2}{k_m},\ldots,\log\frac{k_{m-1}}{k_m}\right).\tag{8}
\end{equation}
Common scaling changes neither normalized mass nor ratios. The chart therefore belongs to the count ray, not its chosen representative.
\begin{corollary}[Ternary strip specialization]\label{cor:ternary}
For $m=2$, equivalently three categories, there are no additional real coordinates and $\Phi_2(\mathbf p)=p_0+i\log(p_1/p_2)$ is a real-analytic diffeomorphism $\D_2^\circ\to\Strip$. For $m>2$, the full geometry is retained as a strip coordinate and $m-2$ additional real contrasts.
\end{corollary}
A complex coordinate consumes exactly two real degrees of freedom: one mass and one oriented contrast. This is why the full ternary simplex gives a strip, whereas the binary interval gives a line under $R_1$.

The construction depends on the distinguished coordinate, reference component, and first contrast. Permuting categories produces equivalent choices of chart. A preferred permutation-equivariant coordinate construction, if needed by an application, is a separate question; no canonical choice is asserted.

\section{One-versus-rest DCCQ map and characteristic lines}\label{sec:ovr}
The barycenter of $\D_m$ is the vector $\mathbf b_m\coloneqq\left(\frac{1}{m+1},\ldots,\frac{1}{m+1}\right)$. Choose one category and constrain the remaining $m$ probabilities to be equal:
\begin{equation}\mathbf P_m(p)\coloneqq\left(p,\frac{1-p}{m},\ldots,\frac{1-p}{m}\right),\qquad0<p<1.\tag{9}\end{equation}
This family has one real degree of freedom: it is a proper submodel for $m>1$ and the full binary model for $m=1$. More precisely, it is a \emph{one-versus-rest Bernoulli submodel with a fixed conditional split}: the indicator of the distinguished category has law $\operatorname{Bernoulli}(p)$, while conditional on the complementary event the remaining $m$ categories are uniformly distributed. The contrast is a shifted Bernoulli logit:
\begin{equation}t_m(p)\coloneqq\log\frac{mp}{1-p}=\log\frac p{1-p}+\log m.\tag{10}\end{equation}
At the barycenter $p=\frac{1}{m+1}$, $t_m=0$. This is the specialization $q_1=\cdots=q_m=1/m$ of the otherwise unrestricted conditional composition in Section~\ref{sec:kernel}.
\begin{definition}[One-versus-rest DCCQ map]
\begin{equation}R_m(p)\coloneqq\frac1{m+1}+i\log\frac{mp}{1-p},\qquad0<p<1.\tag{11}\end{equation}
\end{definition}
The dependence on $m$ can be seen directly by comparison with the binary map:
\[
R_1(p)=\frac12+i\log\frac{p}{1-p},\qquad
R_m(p)=R_1(p)+\left(\frac{1}{m+1}-\frac12\right)+i\log m.
\]
Thus all $R_m$ share the same logit dependence on $p$; $m$ changes the horizontal anchor and the origin of height. The shift $i\log m$ makes the equal-probability point correspond to zero height.

\begin{theorem}[Barycentrically anchored vertical-line theorem]
For every integer $m\ge1$, $R_m$ is a real-analytic bijection from $(0,1)$ onto $L_{\frac{1}{m+1}}\coloneqq\left\{\frac{1}{m+1}+it:t\in\R\right\}$, with inverse
\begin{equation}p=\frac{e^t}{e^t+m}.\tag{12}\end{equation}
The simplex barycenter maps to $s=\frac{1}{m+1}$.
\end{theorem}
\begin{proof}
The binary logit has derivative $\frac{1}{p}+\frac{1}{1-p}>0$ and endpoint limits $-\infty$ and $+\infty$. Translating its image as above gives the stated line for every $m$. Solving $e^t=\frac{mp}{1-p}$ yields the displayed real-analytic inverse.
\end{proof}
The image line $L_{\frac{1}{m+1}}$ will be called the associated DCCQ characteristic line. The anchor $\frac{1}{m+1}$ is motivated by the symmetric point of this slice; the decision to combine it with the contrast is part of the chosen DCCQ coordinate construction.

\paragraph{Full-map restriction, one-versus-rest map, and binary grouping.}
These are different maps. For $m\ge2$, the literal restriction satisfies
\[
\Phi_m(\mathbf P_m(p))=(p+i0,0,\ldots,0),
\]
with no additional real coordinates when $m=2$: all the rest-versus-rest log-ratios vanish. In particular, on the ternary equal-pair family $\mathbf P_2(p)=\left(p,\frac{1-p}{2},\frac{1-p}{2}\right)$,
\[
\begin{aligned}
\Phi_2(\mathbf P_2(p))&=p,\\
R_2(p)&=\frac{1}{3}+i\log\frac{2p}{1-p},\\
R_1(p)&=\frac{1}{2}+i\log\frac{p}{1-p}.
\end{aligned}
\]
The first is the literal restriction of the full coordinate map, the second is the ternary one-versus-rest map, and the third follows after grouping the last two categories into one Bernoulli outcome. Grouping is reversible on this equal-pair family because its conditional split is fixed. Intrinsic dimension one alone does not identify these constructions or force a curve onto $\Re s=1/2$. No regrouping operation is implicit in $\Phi_2$.

\subsection{Reflection, conjugation, and binary case}
Introduce
\begin{equation}C(s)\coloneqq\bar s,\qquad J(s)\coloneqq1-s.\tag{13}\end{equation}
For $s=\sigma+it$, $C(s)=\sigma-it$, $J(s)=1-\sigma-it$, and $J\circ C(s)=1-\sigma+it$. The last is reflection across $\Re s=1/2$. For any real $\sigma$, write $L_\sigma\coloneqq\{\sigma+it:t\in\R\}$.
\begin{proposition}[Symmetry orbit]
For $\sigma\ne1/2$ and $t\ne0$, the orbit generated by $C,J$ is $\{\sigma+it,\sigma-it,1-\sigma+it,1-\sigma-it\}$. Thus $L_\sigma$ and $L_{1-\sigma}$ are exchanged by reflection. For the one-versus-rest family,
\begin{equation}L_{\frac{1}{m+1}}\longleftrightarrow L_{\frac{m}{m+1}}.\tag{14}\end{equation}
\end{proposition}
\begin{theorem}[Conjugation-reflection coincidence]
For $s\in\C$, $C(s)=J(s)$ if and only if $\Re s=1/2$. The binary member $m=1$ is therefore the unique member of the barycentric family whose characteristic line coincides with its reflected partner; both then reduce to the critical line $\Re s=\frac12$.
\end{theorem}
\begin{proof}The equality $\sigma-it=1-\sigma-it$ is equivalent to $2\sigma=1$.\end{proof}
\begin{figure}[ht]\centering
\begin{tikzpicture}[x=13.0cm,y=.72cm,line label/.style={anchor=north,fill=white,inner xsep=2pt,inner ysep=7pt}]
\draw[->] (-0.015,0)--(1.045,0) node[right]{$\Re s$};
\foreach \x/\lab in {0/{0},1/{1}} {
  \draw[gray!65,dashed,line width=.7pt](\x,-2.25)--(\x,2.15);
  \node[line label,yshift=-3pt] at (\x,0) {$\lab$};
}
\foreach \x/\lab in {.25/{\frac14},.333333/{\frac13},.666667/{\frac23},.75/{\frac34}} {
  \draw[gray,dashed,line width=.8pt](\x,-2.15)--(\x,1.85);
  \node[line label,yshift=-3pt] at (\x,0) {$\lab$};
}
\draw[blue!70!black,very thick](.5,-2.25)--(.5,2.15);
\node[line label,yshift=-3pt] at (.5,0) {$\frac12$};
\node[above=2pt] at (.5,2.15) {$m=1$};
\node[above=2pt] at (.333333,1.55) {$m=2$};
\node[above=2pt] at (.666667,1.55) {reflected};
\path[use as bounding box] (-.02,-2.55) rectangle (1.045,2.65);
\end{tikzpicture}
\caption{One-versus-rest vertical lines inside the open critical strip. The dashed outer lines mark the strip boundaries $\Re s=0$ and $\Re s=1$. The solid central line $\Re s=\frac12$ is the reflection axis and the critical line, corresponding to the binary case $m=1$. For $m>1$, the one-versus-rest line $\Re s=\frac{1}{m+1}$ and its reflected partner $\Re s=\frac{m}{m+1}$ are distinct. This is coordinate geometry, not a claim about zeros.}
\end{figure}
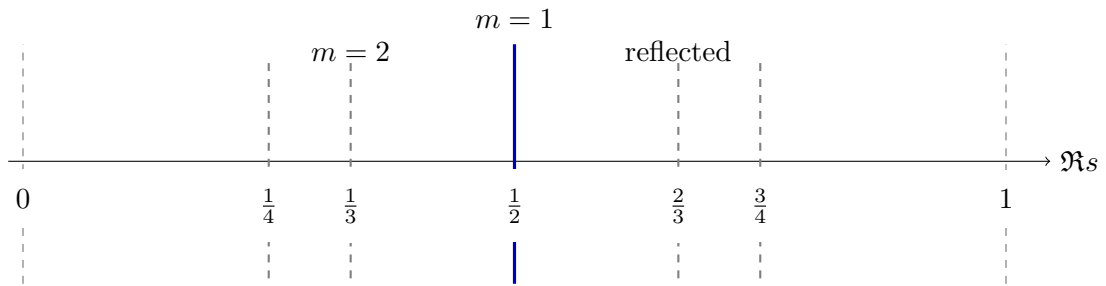

\section{The ternary simplex and the full critical strip}\label{sec:ternary}
The ternary simplex $\D_2$ has exactly two real degrees of freedom and no residual contrast coordinates. This section writes Corollary~\ref{cor:ternary} explicitly; it introduces no new map or independent diffeomorphism theorem. Relabel $p_1=p_+$ and $p_2=p_-$, so that
\[
\D_2^\circ=\{\mathbf p=(p_0,p_+,p_-):p_0,p_+,p_->0,\ p_0+p_++p_-=1\}.
\]
The labels are semantic: $p_0$ is the bounded mass/real-coordinate role, while $p_+,p_-$ form the oriented contrast. They are probabilities, not signed masses. Exchanging $p_+$ and $p_-$ complex-conjugates the coordinate.

The full coordinate map specializes to
\begin{equation}\Phi_2(\mathbf p)=p_0+i\log\frac{p_+}{p_-}.\tag{15}\label{eq:ternary}\end{equation}
At the ternary barycenter $\mathbf b_2=(1/3,1/3,1/3)$, this gives $\Phi_2(\mathbf b_2)=1/3$, so the equal-probability point lies at zero height on the $1/3$ line.
By Theorem~\ref{thm:full}, this is a real-analytic diffeomorphism $\D_2^\circ\to\Strip$. For $s=\sigma+it$, the $m=2$ case of \eqref{eq:inverse} reads
\begin{equation}p_0=\sigma,\quad p_+=(1-\sigma)\frac{e^t}{1+e^t},\quad p_-=(1-\sigma)\frac{1}{1+e^t}.\tag{16}\end{equation}
The inverse is positive, sums to one, and gives $p_+/p_-=e^t$. Both directions are real analytic. As a local check, in variables $(p_0,p_+)$ with $p_-=1-p_0-p_+$, the Jacobian determinant of $(\sigma,t)$ is $\frac{1}{p_+}+\frac{1}{p_-}>0$; Appendix~\ref{app:jacobian} displays the full matrix.

This is related to standard log-ratio ideas but is not an isometric log-ratio (ilr) transformation: its real component is the mass $p_0$, while its imaginary component is one log-ratio. Every fixed-mass segment $p_++p_-=1-\sigma$ maps to the entire vertical line $\Re s=\sigma$. The $\frac{1}{3}$ line is a convenient barycentric example; $\frac{1}{2}$ is independently distinguished by reflection. Figure~\ref{fig:ternary-strip} visualizes this ternary-simplex-to-strip correspondence.

\begin{figure}[htbp]
\centering
\resizebox{\textwidth}{!}{%
\begin{tikzpicture}[x=1.05cm,y=1.05cm, line cap=round, line join=round]
  \begin{scope}[shift={(0,0)}]
    \coordinate (V0) at (0,2.35);
    \coordinate (Vp) at (-1.68,0);
    \coordinate (Vm) at (1.68,0);
    \draw[thick] (Vp)--(Vm)--(V0)--cycle;
    \node[above] at (V0) {$p_0=1$};
    \node[below left] at (Vp) {$p_+=1$};
    \node[below right] at (Vm) {$p_-=1$};
    \node at (0,-0.50) {$\D_2^\circ$};

    \coordinate (L) at (-0.72,1.00);
    \coordinate (R) at (0.72,1.00);
    \fill[gray!18] (L)--(R)--(Vm)--(Vp)--cycle;
    \draw[very thick] (L)--(R);
    \node[above left] at (-0.70,1.05) {$p_0=\sigma$};
    \fill (0,0.63) circle (1.45pt);
    \node[left] at (-0.06,0.63) {$\mathbf b_2$};
    \node[below right] at (0.07,0.60) {$\left(\frac13,\frac13,\frac13\right)$};
  \end{scope}

  \draw[->, thick] (2.15,0.98)--(3.95,0.98) node[midway, above] {$\Phi_2$};

  \begin{scope}[shift={(4.65,0)}]
    \draw[dashed] (0,-1.30)--(0,2.00);
    \draw[dashed] (3.85,-1.30)--(3.85,2.00);
    \draw[->] (-0.45,0)--(4.55,0) node[right] {$\Re s$};
    \draw[->] (1.95,-1.60)--(1.95,2.25) node[above] {$\Im s$};
    \node[below] at (0,0) {$0$};
    \node[below] at (3.85,0) {$1$};
    \node[above] at (1.95,1.95) {$\Strip$};

    \draw[very thick] (1.30,-1.30)--(1.30,2.00);
    \node[below] at (1.30,0) {$\sigma$};
    \node[right] at (1.36,1.42) {$\Re s=\sigma$};
    \fill (1.30,0) circle (1.45pt);
    \node[below right] at (1.37,-0.08) {$\frac13+i0$};

    \draw[<->] (2.75,-0.90)--(2.75,0.90);
    \node[right] at (2.75,0) {$t$};
    \node[align=center] at (2.75,-1.10) {vary $p_+/p_-$};
  \end{scope}
\end{tikzpicture}%
}
\caption{The ternary simplex $\D_2^\circ$ and its DCCQ coordinate image in the full critical strip $\Strip$. A fixed-$p_0=\sigma$ slice in the simplex maps under $\Phi_2$ to the vertical line $\Re s=\sigma$. Motion along the slice changes the ratio $p_+/p_-$, hence the height $t$. The barycenter $(1/3,1/3,1/3)$ maps to $1/3+i0$.}
\label{fig:ternary-strip}
\end{figure}
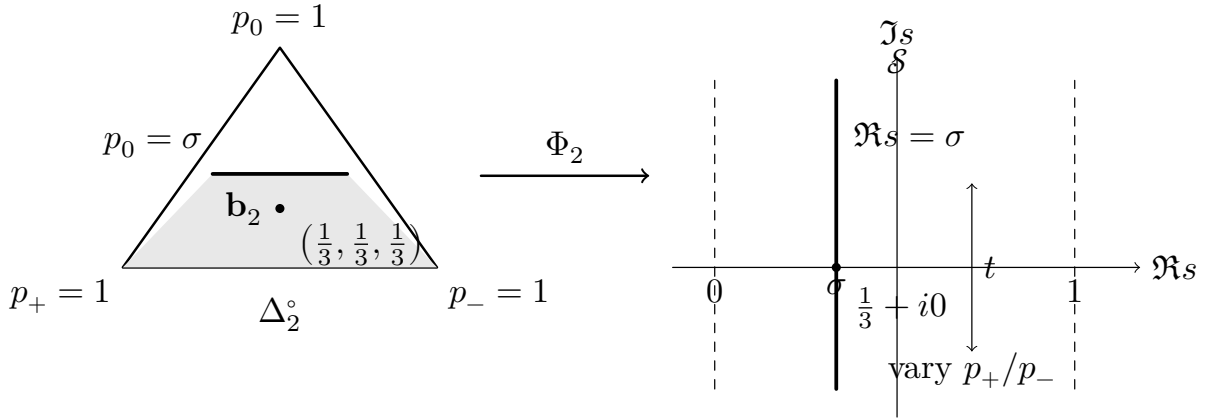

For $m>2$, $\D_m^\circ\cong\Strip\times\R^{m-2}$ retains all degrees of freedom. Only for the ternary member does the residual real factor disappear.

\section{Integer-lattice realization of the ternary strip coordinate}\label{sec:lattice}
The rational simplex has a second exact realization as normalized lattice areas. Let $T\coloneqq\operatorname{conv}(\mathbf v_0,\mathbf v_+,\mathbf v_-)$ be a nondegenerate triangle with vertices $\mathbf v_0,\mathbf v_+,\mathbf v_-\in\mathbb Z^2$, and let $\mathbf x\in T^\circ\cap\mathbb Z^2$. Define
\[
A_T\coloneqq2\operatorname{Area}(T)=|\det(\mathbf v_+-\mathbf v_0,\mathbf v_--\mathbf v_0)|\in\N,
\]
and let $A_0,A_+,A_-$ denote the doubled subareas opposite the corresponding vertices. The symbols $A_j$ here denote \emph{doubled} triangle areas; $A_H$ below denotes an ordinary sector area.
\begin{proposition}[Integer areal coordinates]
$A_0,A_+,A_-\in\N$, $A_0+A_++A_-=A_T$, and
\begin{equation}p_0\coloneqq\frac{A_0}{A_T},\quad p_+\coloneqq\frac{A_+}{A_T},\quad p_-\coloneqq\frac{A_-}{A_T}.\tag{17}\end{equation}
Consequently,
\begin{equation}\Phi_2(\mathbf p)=\frac{A_0}{A_T}+i\log\frac{A_+}{A_-}.\tag{18}\end{equation}
\end{proposition}
\begin{proof}
Doubled lattice-triangle areas are integer determinants. Areal barycentric coordinates give the ratios, and subdivision gives the sum. Total area cancels from the log-ratio.
\end{proof}
\begin{theorem}[Standard count-to-lattice realization]
Let $\mathbf k\coloneqq(k_0,k_+,k_-)$ have positive integer components with $n\coloneqq k_0+k_++k_-$. Define $T_n\coloneqq\operatorname{conv}\{(0,0),(n,0),(0,n)\}$ and $\mathbf x_{\mathbf k}\coloneqq(k_+,k_-)$. Then
\begin{equation}A_{T_n}=n^2,\qquad(A_0,A_+,A_-)=n(k_0,k_+,k_-),\tag{19}\end{equation}
and hence
\begin{equation}\Phi_2([\mathbf k])=\frac{k_0}n+i\log\frac{k_+}{k_-}.\tag{20}\end{equation}
The count and doubled-area vectors represent the same positive ray class.
\end{theorem}
\begin{proof}
The parent doubled area is $n^2$. The barycentric coordinates of $\mathbf x_{\mathbf k}$ relative to the displayed vertices are $(k_0/n,k_+/n,k_-/n)$. Multiplying by $n^2$ gives the subareas. In the coordinate calculation the common factors cancel:
\[
\frac{A_0}{A_{T_n}}+i\log\frac{A_+}{A_-}
=\frac{nk_0}{n^2}+i\log\frac{nk_+}{nk_-}
=\frac{k_0}{n}+i\log\frac{k_+}{k_-}.\qedhere
\]
\end{proof}
For the same ternary ray $\mathbf k=(2,3,1)$, one has $n=6$ and $(A_0,A_+,A_-)=(12,18,6)$; both count and area coordinates therefore give $\Phi_2=1/3+i\log 3$.

Scaling counts by an integer $g$ scales the standard triangle's lengths by $g$ and its areas by $g^2$, while $\Phi_2$ stays fixed. Thus this realization does not preserve absolute lattice area under the count-ray quotient. Direct count substitution computes the coordinate without an area calculation; the lattice theorem supplies its exact geometric realization.

In dimension $m$, lattice-normalized volume is $m!\operatorname{Vol}_m$. For $\operatorname{conv}(\mathbf 0,n\mathbf e_1,\ldots,n\mathbf e_m)$ and a composition $\sum_{j=0}^mk_j=n$, the normalized subvolumes are $n^{m-1}k_j$ \cite{beck}. The planar doubled-area formula is its $m=2$ instance.
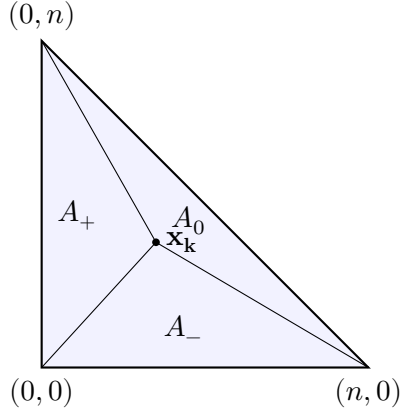
\begin{figure}[ht]\centering
\begin{tikzpicture}[scale=.72]
\fill[blue!5](0,0)--(6,0)--(0,6)--cycle;
\draw[thick](0,0)node[below]{$(0,0)$}--(6,0)node[below]{$(n,0)$}--(0,6)node[above]{$(0,n)$}--cycle;
\draw(0,0)--(2.1,2.3)--(0,6);\draw(2.1,2.3)--(6,0);
\fill(2.1,2.3)circle(2pt);\node[right]at(2.1,2.3){$\mathbf x_{\mathbf k}$};
\node at(.65,2.8){$A_+$};\node at(2.7,2.7){$A_0$};\node at(2.6,.7){$A_-$};
\end{tikzpicture}
\caption{Standard lattice realization, with $\mathbf x_{\mathbf k}=(k_+,k_-)$. Doubled subareas are $n(k_0,k_+,k_-)$, so areal normalization agrees with count normalization.}
\end{figure}

\FloatBarrier
\section{Hyperbolic representation of the ternary log-ratio}\label{sec:hyperbolic}
This section introduces no new degree of freedom; it gives a hyperbolic geometric representation of the ternary log-ratio $t$.

For a general ternary probability vector, the vertical coordinate is $t\coloneqq\log\frac{p_+}{p_-}$. When the vector has a positive integer count or lattice-area realization, the same quantity is
\begin{equation}t=\log\frac{p_+}{p_-}=\log\frac{k_+}{k_-}=\log\frac{A_+}{A_-}.\tag{21}\end{equation}
Introduce
\begin{equation}h\coloneqq\frac t2.\tag{22}\end{equation}
Then
\begin{equation}\gamma(h)\coloneqq(\cosh h,\sinh h)\tag{23}\end{equation}
lies on $\mathcal H_+\coloneqq\{(X,Y):X^2-Y^2=1,\ X>0\}$ by the standard hyperbolic identities \cite{dlmf}. For integer representatives,
\begin{equation}X\coloneqq\frac{k_++k_-}{2\sqrt{k_+k_-}},\qquad Y\coloneqq\frac{k_+-k_-}{2\sqrt{k_+k_-}}.\tag{24}\end{equation}
These formulas also hold with counts replaced by doubled subareas. Positive integer pair-ray classes are identified with $\Q_{>0}$ by their ratio $r\coloneqq \frac{k_+}{k_-}$. Thus $h=\frac12\log r$ maps them bijectively to the dense subset $\frac12\log\Q_{>0}\subset\R$. Extending $r$ to $\R_{>0}$ gives the full real parameter axis and hence the whole right branch. This is a density-and-continuity extension, not an assertion that the integer set is already the continuous branch.

\paragraph{The reversible geometric object.}
The inverse ratio is $r=\frac{X+Y}{X-Y}$; the denominator is positive on $\mathcal H_+$. For the same ternary ray used above, $r=3$ and $(X,Y)=(2/\sqrt3,1/\sqrt3)$, so the inverse formula returns $3$. Exchanging $k_+$ and $k_-$ leaves $X$ fixed, sends $Y$ to $-Y$, and sends $r$ to $1/3$, making the orientation information explicit. Retaining the mass coordinate gives the precise identification
\[
\D_2^\circ\cong(0,1)\times\mathcal H_+,\qquad
\mathbf p\longmapsto(p_0,\gamma(\frac12\log(\frac{p_+}{p_-}))).
\]
Its inverse is
\[
p_0=\sigma,\qquad p_+=(1-\sigma)\frac{X+Y}{2X},\qquad
p_-=(1-\sigma)\frac{X-Y}{2X}.
\]
A fixed-mass slice corresponds to one full hyperbola branch. The whole simplex needs both factors. Retaining only $X=\cosh h$ loses the distinction between $h$ and $-h$; the pair $(X,Y)$, or an additional orientation label, is needed for signed recovery. These are real-analytic coordinate identifications, not claims of lattice or metric isometry.

\paragraph{Which area is represented?}
Let $\mathbf O\coloneqq(0,0)$, $\mathbf E\coloneqq(1,0)$, and $\mathbf P\coloneqq\gamma(h)$. Let $A_H(h)$ be the signed ordinary planar area bounded by the segment $\mathbf O\mathbf E$, the hyperbola arc from $\mathbf E$ to $\mathbf P$, and the radial segment $\mathbf P\mathbf O$. Then
\begin{equation}A_H(h)=\frac h2=\frac t4.\tag{25}\end{equation}
Indeed, along $\gamma(v)$, $X\,dY-Y\,dX=dv$, while radial segments contribute zero. The oriented planar area formula therefore gives $2A_H=\int_0^h dv=h$, the signed sector interpretation derived directly here. In particular, the same contrast has the three normalizations
\[
t=2h=4A_H.
\]
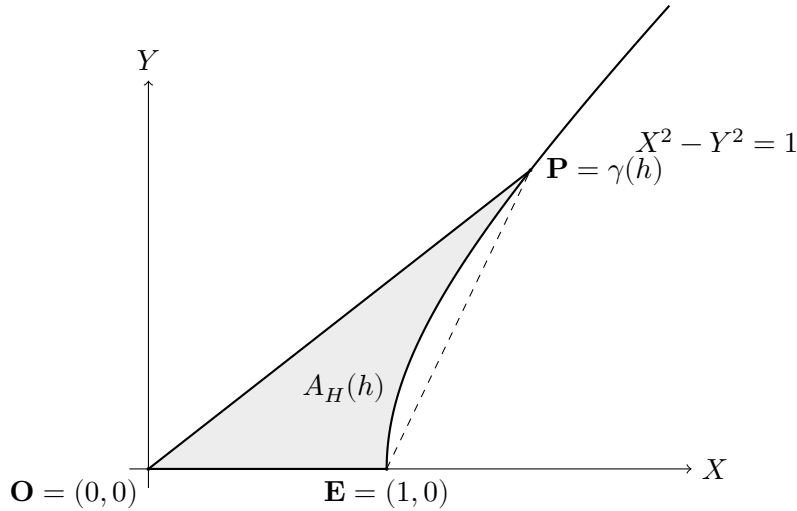
\begin{figure}[htbp]
\centering
\begin{tikzpicture}[x=3.15cm,y=3.15cm]
\def\hval{1.05}
\coordinate (O) at (0,0);
\coordinate (E) at (1,0);
\coordinate (P) at ({(exp(\hval)+exp(-\hval))/2},{(exp(\hval)-exp(-\hval))/2});
\path[fill=black!7]
  (O)--(E)
  plot[domain=0:\hval,samples=90,variable=\u]
    ({(exp(\u)+exp(-\u))/2},{(exp(\u)-exp(-\u))/2})
  --(O)--cycle;
\draw[->] (-0.08,0)--(2.28,0) node[right] {$X$};
\draw[->] (0,-0.08)--(0,1.63) node[above] {$Y$};
\draw[thick] (O)--(E);
\draw[thick,domain=0:1.42,samples=110,variable=\u]
  plot ({(exp(\u)+exp(-\u))/2},{(exp(\u)-exp(-\u))/2});
\draw[thick] (O)--(P);
\draw[dashed] (E)--(P);
\fill (O) circle (0.7pt) node[below left] {$\mathbf O=(0,0)$};
\fill (E) circle (0.7pt) node[below] {$\mathbf E=(1,0)$};
\fill (P) circle (0.7pt) node[right=2pt] {$\mathbf P=\gamma(h)$};
\node at (0.82,0.34) {$A_H(h)$};
\node[anchor=west] at (2.00,1.38) {$X^2-Y^2=1$};
\end{tikzpicture}
\caption{Signed planar sector $A_H(h)$, illustrated for $h>0$. The sector is bounded by the straight segments $\mathbf O\mathbf E$ and $\mathbf O\mathbf P$ and by the hyperbola arc from $\mathbf E$ to $\mathbf P=\gamma(h)$. The dashed chord $\mathbf E\mathbf P$ marks the straight triangle $\mathbf O\mathbf E\mathbf P$ for comparison.}
\label{fig:hyperbolic-sector}
\end{figure}
Consequently,
\begin{equation}\Phi_2(\mathbf p)=p_0+4iA_H.\tag{26}\end{equation}
Equation (26) re-expresses the existing contrast; it adds no independent degree of freedom. The lattice data enter by the area-ratio transformation
\[
A_H=\frac14\log\frac{A_+}{A_-}.
\]
The sector area is neither a doubled lattice-triangle area nor the Riemannian area measure of a hyperbolic plane. The straight triangle $\mathbf O\mathbf E\mathbf P$ has signed area $\sinh(h)/2$, whereas the curved sector has area $h/2$. The occurrence of a factor two in both settings follows from the same planar determinant/integral area formula; it does not identify their regions or make the correspondence area-preserving.

Exchanging $k_+,k_-$ sends $t,h,A_H$ to their negatives and $(X,Y)$ to $(X,-Y)$. This is the endpoint of the hyperbolic construction used in the present paper. The even coordinate $X=\cosh(t/2)$ is a natural object for further analytic study, which lies outside the present scope.

\section{Transverse symmetry defect}\label{sec:defect}
The hyperbolic parameter runs along a vertical line. A separate unsigned scalar measures the transverse mismatch between the conjugation image $C(s)$ and the reflection image $J(s)$. Here ``defect'' refers to the separation of these two symmetry images, not to a failure of either symmetry itself.
\begin{definition}[Symmetry defect]
\begin{equation}D(s)\coloneqq|(1-s)-\bar s|,\qquad s=\sigma+it.\tag{27}\end{equation}
\end{definition}
\begin{proposition}
\begin{equation}D(s)=|1-2\sigma|=2\left|\sigma-\frac12\right|.\tag{28}\end{equation}
Thus $D(s)=0$ exactly on $\Re s=1/2$.
\end{proposition}
Thus a symmetry collision, $C(s)=J(s)$, occurs exactly when $D(s)=0$. Because $0<\sigma<1$ in the open strip,
\[
0\le D(s)<1.
\]
The minimum $0$ is attained precisely on the critical line, while $1$ is the unattained supremum approached as $\sigma$ tends to either strip boundary. The height cancels because $(1-s)-\bar s=1-2\sigma$. For the barycentric family, the ordinary distance of the line from the reflection axis is
\begin{equation}a_m\coloneqq\left|\frac1{m+1}-\frac12\right|=\frac{m-1}{2(m+1)},\qquad m\ge1,\tag{29}\end{equation}
so
\begin{equation}D_m\coloneqq2a_m=\frac{m-1}{m+1}.\tag{30}\end{equation}
Hence $D_1=0$, $D_2=1/3$, $D_3=1/2$, and $D_m\uparrow1$ as $m\to\infty$; the one-versus-rest characteristic lines therefore approach the strip boundary in this normalized defect scale.
In the ternary strip, $\sigma=p_0$ and $D=|1-2p_0|$. The reflection axis is the slice $p_0=1/2$. This unsigned transverse measure is independent of $t$ and does not record which side of the critical line contains the point. The quantities $D_m$ and $a_m$ encode the same displacement, with $D_m=2a_m$. In the translation formula of Section~\ref{sec:ovr}, this gives $R_m(p)=R_1(p)-a_m+i\log m$. Here $D$ denotes symmetry defect, not any doubled-area variable.

\section{Relation to standard critical-strip symmetry}\label{sec:xi}
The completed Riemann xi function satisfies
\begin{equation}\xi(s)=\xi(1-s),\tag{31}\end{equation}
together with conjugation symmetry. A non-real zero away from the axis therefore has reflected and conjugate partners \cite{edwards,titchmarsh}. The Riemann Hypothesis asserts that all nontrivial zeros lie on $\Re s=1/2$.

The construction here is a coordinate model. Writing $R_2(p)=1/3+it$ gives a coordinate identity; asserting that $\rho=1/3+it$ is a zero would be a different statement. Neither the lattice realization nor the hyperbolic representation identifies a DCCQ object with $\xi$.

Reflection symmetry alone does not force a point onto the axis: any $L_\sigma$ can be paired with $L_{1-\sigma}$. The special property of $L_{1/2}$ is its coincidence with its reflected partner and, pointwise, $C(s)=J(s)$. A mechanism selecting that axis as a zero-bearing set needs analytic information beyond the coordinate symmetry.

\section{Positioning, scope, and limitations}\label{sec:scope}
\paragraph{Established ingredients.}
Multinomial kernels and probability simplices \cite{cover}, closure and log-ratio coordinates \cite{aitchison,egozcue}, integer lattice compositions and ray projection \cite{lovell}, barycentric areas and normalized lattice volumes \cite{beck}, hyperbolic identities \cite{dlmf} and the elementary planar sector-area formula, and zeta/xi symmetries \cite{edwards,titchmarsh} are established. The underlying real tuple $(p_0,u_1,\ldots,u_{m-1})$ is a standard information-geometric mixed coordinate system \cite{amari,sei}. No additional Fisher-metric, orthogonality, curvature, or isometry claim is inferred here from its subsequent complex coordinate construction unless explicitly stated.

\paragraph{Proposed DCCQ formulation.}
The proposed contribution is their organization into
\[
\begin{aligned}
\text{integer count rays}&\longrightarrow\text{multinomial simplex}\\
&\longrightarrow\text{complex coordinate construction}\\
&\longrightarrow\text{vertical lines and reflection geometry}.
\end{aligned}
\]
The specific maps are $R_m(p)=\frac{1}{m+1}+i\log\frac{mp}{1-p}$ and $\Phi_m:\D_m^\circ\to\Strip\times\R^{m-2}$, with the ternary formula $\Phi_2(\mathbf p)=p_0+i\log(\frac{p_+}{p_-})$. The standard lattice triangle gives $\Phi_2([\mathbf k])=k_0/n+i\log(\frac{k_+}{k_-})$ on rational compositions. The general integer formula retains every rational composition coordinate; rational density and continuity connect this discrete set with the full real model.

The cited literature supplies the constituent ingredients. The specific complex-strip coordinate construction and its reflection/conjugation interpretation are presented here as the DCCQ formulation; no exhaustive priority claim is made.

\paragraph{Claims not made.}
The paper does not assert zeros on $\Re s=\frac{1}{m+1}$ for $m>1$, additional critical lines in standard number theory, novelty of closure or sector area, equality of a hyperbolic coordinate with $\xi$, or a proof of RH. A countable integer-ray set is not identified with the entire continuous simplex or hyperbola. Neither lattice area nor an unspecified metric is preserved by the coordinate correspondence.

\section{Discussion and further directions}\label{sec:discussion}
The barycentric family supplies $a_1=0$, $a_2=1/6$, $a_3=1/4$, $a_4=3/10$, and $a_m\uparrow1/2$. The full ternary model supplies every $a\coloneqq|p_0-1/2|\in[0,1/2)$. Thus there is both an alphabet-indexed sequence of unsigned offsets and a continuous range of such offsets in the strip.

For $m>2$, $\D_m^\circ\cong\Strip\times\R^{m-2}$ resolves the dimensional bookkeeping. A later application may select a preferred strip coordinate or a permutation-equivariant atlas. The group $S_{m+1}$ acts on the choices of distinguished category, reference, and contrast, while the barycenter stays fixed.

The normalized-volume realization $\mathbf k\mapsto n^{m-1}\mathbf k$ gives the same ray. It realizes exactly the chart image of rational interior compositions, a dense subset of $\Strip\times\R^{m-2}$; allowing real compositions gives the full target. This distinction separates integer realizability from continuous parameterization.

Finally, the framework supplies a reusable interface for arithmetic observables and analytic transforms. Extensions must specify what scale information they retain, which observable they transform, and what mechanism selects a spectral set. None of these choices is fixed by $\Phi_m$ itself. The coordinate theory developed here is self-contained and requires no additional spectral construction.

\section{Conclusion}
Alphabet expansion changes the binary interval into a multinomial simplex. Integer compositions modulo common scaling reconstruct the rational simplex, whose Euclidean closure is the closed simplex. A primitive representative records the rational ray, and its gcd multiplier restores the original integer scale.

The dimensional hierarchy is explicit: $m=1$ gives one critical line under the binary chart; $m=2$ gives the full strip; and $m>2$ gives $\Strip\times\R^{m-2}$. In $\Phi_m$, one bounded mass and one oriented contrast form a complex coordinate, while all remaining contrasts remain real coordinates. The one-versus-rest family is a Bernoulli submodel with a fixed conditional split. Its map $R_m$, the literal restriction of the full coordinate map, and the grouped binary map are distinct operations. Only the binary characteristic line coincides with its reflected partner.

In the ternary case, semantic labels give $\Phi_2(\mathbf p)=p_0+i\log\frac{p_+}{p_-}$. Standard lattice subareas give the same rational point and the same log-ratio. The hyperbolic representation retains the mass and oriented branch point, $\D_2^\circ\cong(0,1)\times\mathcal H_+$, with $t=4A_H$ for ordinary signed planar sector area. It uses ratios of lattice areas rather than preserving their absolute values. These exact interfaces strengthen the discrete-to-multinomial-to-complex coordinate framework while keeping analytic zero selection as a separate problem.

\section*{Declaration of generative AI use}
The author conceived the central ideas and framework, directed the development of the work, and made all final scientific and editorial decisions. ChatGPT (OpenAI) was used as an AI research assistant for mathematical exploration and formulation, literature discovery, technical review, consistency checking, testing, revision, and language editing. The author critically reviewed and verified the mathematics, citations, and final text and takes full responsibility for the publication.

\appendix
\section{Auxiliary calculations}
\subsection{Jacobian of the ternary strip map}\label{app:jacobian}
With $p_-=1-p_0-p_+$,
\[
\frac{\partial(\sigma,t)}{\partial(p_0,p_+)}=
\begin{pmatrix}1&0\\[4pt]\frac{1}{p_-}&\frac{1}{p_+}+\frac{1}{p_-}\end{pmatrix},
\qquad \det=\frac{1}{p_+}+\frac{1}{p_-}>0.
\]
\subsection{Lattice subareas in the standard triangle}
For $T_n\coloneqq\operatorname{conv}\{(0,0),(n,0),(0,n)\}$ and $\mathbf x\coloneqq(k_+,k_-)$, the doubled parent area is $n^2$. The subtriangle opposite $(n,0)$ has doubled area $nk_+$, that opposite $(0,n)$ has $nk_-$, and the remaining one has $n(n-k_+-k_-)=nk_0$.

\section{Notation reference}\label{app:notation}
The table distinguishes the dimension index $m$ from the count/length scale $n$, and the full map $\Phi_m$ from the one-dimensional map $R_m$. It collects notation; no additional assumptions are imposed here.
\begingroup
\small
\renewcommand{\arraystretch}{1.17}
\setlength{\LTpre}{4pt}
\setlength{\LTpost}{4pt}
\begin{longtable}{@{}>{\raggedright\arraybackslash}p{.23\linewidth}>{\raggedright\arraybackslash}p{.60\linewidth}>{\raggedright\arraybackslash}p{.11\linewidth}@{}}
\toprule
Symbol & Meaning & Section \\
\midrule
\endfirsthead
\toprule
Symbol & Meaning & Section \\
\midrule
\endhead
\midrule
\multicolumn{3}{r@{}}{Continued on the next page} \\
\endfoot
\bottomrule
\endlastfoot
$m$ & Probability dimension; $m+1$ categories are labeled $0,\ldots,m$. & \ref{sec:construction} \\
$\mathbf p=(p_0,\ldots,p_m)$ & Probability vector; $p_j\ge0$ and $\sum_jp_j=1$. Boldface denotes the vector, not its scalar components. & \ref{sec:kernel} \\
$\mathbf k=(k_0,\ldots,k_m)$ & Integer count vector; nonzero when normalized. & \ref{sec:kernel} \\
$n$ & Total count $\sum_jk_j$; also the side-length scale of the standard lattice triangle $T_n$. & \ref{sec:kernel}, \ref{sec:lattice} \\
$B(p;n,k)$ & Ordered binary-word kernel $p^k(1-p)^{n-k}$. & \ref{sec:intro} \\
$M(\mathbf p;\mathbf k)$ & Ordered multinomial word kernel $\prod_jp_j^{k_j}$, without the multiplicity coefficient. & \ref{sec:kernel} \\
$q_j$ & Conditional probability $p_j/(1-p_0)$ for $j=1,\ldots,m$; these sum to one. & \ref{sec:kernel} \\
$u_j$ & Conditional log-ratio $\log(q_j/q_m)=\log(p_j/p_m)$, $j=1,\ldots,m-1$. & \ref{sec:coordinate} \\
$Z$ & Derived normalizer $1+\sum_{j=1}^{m-1}e^{u_j}$; not a free parameter. & \ref{sec:coordinate} \\
$\D_m,\ \D_m^\circ$ & Closed probability simplex in $\R^{m+1}$ and its relative interior (all components positive). & \ref{sec:simplex} \\
$\mathcal{C}_m,\ \sim,\ [\mathbf k]$ & Nonzero integer count vectors, equivalence by equal normalized composition, and a count-ray class. & \ref{sec:simplex} \\
$\mathbf r$ & A rational simplex point in the reconstruction proof. & \ref{sec:simplex} \\
$\mathbf k_{\rm prim},\ g$ & Primitive count vector and its gcd multiplier, with $\mathbf k=g\mathbf k_{\rm prim}$. & \ref{sec:simplex} \\
$\Phi_m$ & Full coordinate map to $\Strip\times\R^{m-2}$ for $m\ge2$. On counts, $\Phi_m([\mathbf k])$ means $\Phi_m(\mathbf k/n)$. & \ref{sec:coordinate}, \ref{sec:chart} \\
$\Strip$ & Open critical strip $\{s\in\C:0<\Re s<1\}$, viewed as a real two-dimensional coordinate domain. & \ref{sec:intro} \\
$s=\sigma+it$ & Complex strip coordinate. In $\Phi_m$, $\sigma=p_0$ and $t=u_1$. & \ref{sec:chart} \\
$\mathbf b_m$ & Equal-probability barycenter $(1/(m+1),\ldots,1/(m+1))$. & \ref{sec:ovr} \\
$\mathbf P_m(p)$ & One-versus-rest embedding $\left(p,(1-p)/m,\ldots,(1-p)/m\right)$. & \ref{sec:ovr} \\
$t_m(p),\ R_m(p)$ & Shifted logit $t_m=\log(mp/(1-p))$ and one-versus-rest map $R_m=1/(m+1)+it_m$. & \ref{sec:ovr} \\
$L_\sigma$ & Vertical line $\{\sigma+it:t\in\R\}$. & \ref{sec:ovr} \\
$C,\ J$ & Involutions $C(s)=\bar s$ and $J(s)=1-s$; $J\circ C$ reflects across $\Re s=1/2$. & \ref{sec:ovr} \\
$p_+,\ p_-$ & Ternary labels for $p_1,p_2$; positive probabilities forming an oriented log-ratio, not signed masses. & \ref{sec:ternary} \\
$T,\ T_n$ & A lattice triangle and the standard one $\operatorname{conv}\{(0,0),(n,0),(0,n)\}$. These subscripts specify scale, not dimension. & \ref{sec:lattice} \\
$\mathbf v_j,\mathbf e_j,\mathbf x_{\mathbf k}$ & Lattice vertices, standard basis vectors, and the point $\mathbf x_{\mathbf k}=(k_+,k_-)$. & \ref{sec:lattice} \\
$A_T,\ A_0,A_+,A_-$ & Doubled parent and opposite subtriangle areas. They satisfy $A_0+A_++A_-=A_T$. & \ref{sec:lattice} \\
$r,\ h$ & Positive contrast ratio and hyperbolic parameter $h=\frac12\log r=t/2$. & \ref{sec:hyperbolic} \\
$\gamma(h),\ (X,Y)$ & Hyperbola parameterization $(\cosh h,\sinh h)$ and its two scalar coordinates. & \ref{sec:hyperbolic} \\
$\mathcal H_+$ & Right branch $X^2-Y^2=1$, $X>0$; both $X$ and $Y$ are retained for orientation. & \ref{sec:hyperbolic} \\
$\mathbf O,\mathbf E,\mathbf P$ & Sector points $(0,0)$, $(1,0)$, and $\gamma(h)$. & \ref{sec:hyperbolic} \\
$A_H$ & Signed ordinary planar sector area; $A_H=h/2=t/4$, not a doubled lattice area or a hyperbolic-plane area measure. & \ref{sec:hyperbolic} \\
$D(s),\ D_m,\ a_m,\ a$ & Unsigned symmetry defect $D(s)=|1-2\Re s|$; $D_m=2a_m$ on the selected lines, and $a=|p_0-1/2|$ for a general ternary point. & \ref{sec:defect} \\
$\zeta,\ \xi,\ \rho$ & Riemann zeta function, completed xi function, and a zero when discussing standard zero symmetry. & \ref{sec:xi} \\
$S_{m+1}$ & Permutation group on the $m+1$ categories; distinct from the strip $\Strip$. & \ref{sec:discussion} \\
\end{longtable}
\endgroup

\end{document}